\documentclass[sigconf]{aamas} 
\usepackage{soul}
\usepackage{url}
\usepackage{graphicx}
\usepackage{subcaption}
\usepackage{balance}

\usepackage{hyperref}

\theoremstyle{definition}

\newtheorem{corollary}{Corollary}

\newtheorem{theorem}{Theorem}

\newtheorem{example}{Example}

\usepackage{xcolor}

\newcommand{\judge}{\ensuremath{j} }
\newcommand{\judges}{\ensuremath{J} }
\newcommand{\reps}{\ensuremath{E} }
\newcommand{\rep}{\ensuremath{e} }
\newcommand{\vote}{\ensuremath{v} }

\PassOptionsToPackage{numbers, compress, sort}{natbib}

\title{Distributed Weighting of Experts}

\author{Ben Abramowitz}
\affiliation{
  \institution{Tulane University}
  \city{New Orleans}
  \country{USA}%
}
\author{Nicholas Mattei}
\affiliation{
  \institution{Tulane University}
  \city{New Orleans}
  \country{USA}%
}

\begin{abstract}
    Aggregating signals from a collection of noisy sources is a fundamental problem in many domains including crowd-sourcing, multi-agent planning, sensor networks, signal processing, voting, ensemble learning, and federated learning. The core question is how to aggregate signals from multiple sources (e.g. experts) in order to reveal an underlying ground truth. While a full answer depends on the type of signal, correlation of signals, and desired output, a problem common to all of these applications is that of differentiating sources based on their quality and weighting them accordingly.
    It is often assumed that this differentiation and aggregation is done by a \emph{single}, \emph{accurate} central mechanism or agent (e.g. judge). We complicate this model in two ways. First, we investigate the setting with both a single judge, and one with multiple judges. Second, given this multi-agent interaction of judges, we investigate various constraints on the judges' reporting space.
    We build on known results for the optimal weighting of experts and prove that an ensemble of sub-optimal mechanisms can perform optimally under certain conditions. We then show empirically that the ensemble approximates the performance of the optimal mechanism under a broader range of conditions.
\end{abstract}

\begin{document}
\maketitle

\section{Introduction}
Aggregating noisy information from a group of agents or algorithms into a label or decision is a fundamental problem across many fields. Take the examples of crowd-sourcing image labels for training supervised learning models~\cite{quoc2013evaluation}, fusing conflicting sensor data~\cite{pau1988sensor}, ensemble methods in machine learning~\cite{dietterich2000ensemble}, interactive democracy~\cite{brill2018interactive}, peer review~\cite{lev2021peer}, and even guessing the weight of an ox~\cite{surowiecki2005wisdom}. In each situation there is some underlying ground truth, i.e., the weight of the ox or whether the image contains a tiger. In all these settings we wish to combine a number of weak signals into a single strong signal or decision.
In the simplest cases, all information sources are treated equally, e.g. anonymous voting or uniform weighting of image labels, and aggregation methods depend on some basic notion of centrality, e.g. the mean or median. However, when one can assess the quality or reliability of a signal or its source, significant improvement becomes possible.

For example, in a simplistic model of academic peer review, a conference chair (judge) must determine whether to accept or reject papers without reading them based on the accept/reject statements from reviewers (experts). The chair may reasonably give higher weight to the reports of reviewers who indicate greater expertise \cite{lev2021peer}. Of course, the chair may be inaccurate in how competent they believe each of the reviewers to be.

We base our investigation on the literature on weighting experts in both the offline~\cite{shapley1984optimizing,nitzan1982optimal} and online settings~\cite{cesa1997use,vovk1990aggregating,freeman2020no}, though in this work we restrict our focus to a single decision. A set of independent \emph{experts} (e.g. sensors, agents, or algorithms) seeks to determine a binary ground truth. Each expert has a certain \emph{competence}, or probability of being correct. Each expert can provide a single bit of information (e.g. True or False), but the experts cannot communicate otherwise. If nothing is known about the experts and their competences, and nothing additional is known about the ground truth, the only reasonable way to aggregate these bits is by a majority vote~\cite{may1952set}. As the number of experts increases, as long as they are sufficiently competent, e.g. all competences $> 0.51$, the Condorcet Jury Theorem says the probability of majority voting aggregating correctly tends to one~\cite{berend1998condorcet}.
However, when there are only a few experts the asymptotic behavior is not meaningful, and when enough of the experts are incompetent, e.g. have competence $\leq 0.50$, the theorem no longer holds. Moreover, when the competences of the experts are known, majority rule becomes sub-optimal \cite{berend1998condorcet}.

Fortunately, the optimal aggregation method for maximizing accuracy with any number of independent experts, with any competences, is straightforward~\cite{shapley1984optimizing,nitzan1982optimal}. The optimal method is to give each expert a weight equal to the log-odds of their competence, and then take a weighted majority vote. At first this method would appear to require that the competences of the experts be known. Currently, the only known method of assigning experts their optimal weights is for a central authority, who knows the exact competences, to compute and assign the proper weights. One of our main contributions, detailed in Section \ref{section:optimal}, is a proof that no central authority is required. With multiple judges, no single judge needs to know either the ground truth or the true competence of any of the experts. Just as the experts' votes can be aggregated to achieve higher accuracy than any of the experts individually~\cite{grofman1978judgmental,feld1984accuracy}, aggregating weights from an ensemble of \emph{judges} can be better than any one individually, and under certain conditions achieves the optimal weighting.

Consider an autonomous system with two kinds of sensors. Both sensor types take regular measurements of the same kind (e.g. path obstructed or unobstructed). The first type of sensor is cheap, takes multiple measurements each second, and can transmit a single bit every second, but accuracy is highly variable across sensors. The second type of sensor is more costly, takes a measurement every few seconds, and is more reliable, but can only receive and transmit a few bits each minute. If decisions must be made quickly, the second sensor might seem useless. However, if these slower sensors can be used to judge the accuracy of the faster sensors at regular intervals, the overall accuracy of the entire ensemble may be improved. The same intuition applies to the use of learning algorithms and approximation algorithms that require different amounts of time to compute in time-sensitive applications. More reliable algorithms can be used to evaluate ensembles of faster, less accurate algorithms over time, a technique used in many ensemble solvers for hard computational problems \cite{thornton2013auto}.

Our approach of decentralizing the weighting of experts is inspired by work in ``wisdom of the crowds" and crowdsourcing~\cite{surowiecki2005wisdom,brabham2013crowdsourcing}, proxy voting~\cite{abramowitz2019flexible} and truth-tracking in Liquid Democracy~\cite{zhang2021tracking,becker2021can,becker2021unveiling}.\footnote{See ~\cite{paulin2020overview} for an overview of Liquid Democracy.}
For human agents, it is often more natural for them to assign weights or scores to the experts rather than to report probabilities as estimates of the experts' competences.
We wish to study multi-agent learning models with low communication complexity that are appropriate for human and computational agents alike. Hence, the judges in our model only provide real-valued weights for each expert.

We must also address the impractical nature of the optimal weighting rule. 
The optimal weights are negative for experts whose competence is below $0.5$. In voting, it can be unnatural to allow negative weights, especially since any expert who knows their weight is negative might reverse their vote. Many papers on voting and variants of the Condorcet Jury Theorem assume all experts have competence $> 0.5$, but we do not make this assumption. Rather, we consider the impact on accuracy when weights are required to be non-negative. This effectively removes experts whose weights would be negative rather than negating their votes.
Lastly, the optimal weights can be arbitrarily large (small) when competences approach 1.0 (0.0). In the multi-judge setting, this means that a single judge may dominate any aggregated set of weights. In practice, it may be necessary to assume the weights are in some finite range. We therefore consider the impact on accuracy when the weights judges assign are normalized so that they sum to 1.0 for each judge. This is the equivalent of ``one-person-one-vote" for the judges. As with the experts, if nothing is known about the quality of each judge, treating them all equally may be most reasonable.

\section{Model and Notation}
In our model there are two disjoint sets of agents -- judges and experts. 
Let \reps be a set of $m$ experts and \judges be a set of $n$ judges.
The experts vote on a single binary issue in which there is only one right answer. Without loss of generality, the alternatives are represented by $\{1,0\}$ where $1$ is correct and $0$ is incorrect. 
Each expert $\rep \in \reps$ has a \emph{competence}, or probability $p_\rep$ of voting correctly, independent of all other experts.
We associate each expert's index with their vote, so expert $\rep \in \reps$ casts a vote $\vote_\rep \in \{1,0\}$ with competence $p_\rep = P(\vote_\rep = 1)$. We assume that for every $\rep \in \reps$ the vote of each expert is independent from all other experts.
The \emph{odds} of an expert voting correctly are hence $\frac{p_\rep}{1-p_\rep}$, and their \emph{log-odds} are $\log \left( \frac{p_\rep}{1-p_\rep} \right)$.
For simplicity (and realism), we assume that $0 < p_\rep < 1$ for the experts, meaning that no expert is either always correct or always incorrect.

\paragraph{Weighted Majority Rules for Aggregating Expert Votes}
A weighted majority rule gives each expert a weight $w_\rep \in \mathbb{R}$ and elects $1$ as the winner if $\sum\limits_{\vote_\rep = 1} w_\rep > \sum\limits_{\vote_\rep = 0} w_\rep$, elects $0$ as the winner if $\sum\limits_{\vote_\rep = 1} w_\rep < \sum\limits_{\vote_\rep = 0} w_\rep$, and uses a tie-breaking rule (e.g. coin flip) for the edge case where these sums are equal.
Note that if all experts' weights are scaled up or down by some constant factor, the rule does not change.

\paragraph{Optimal Weighting via the Log-Odds Rule}
The optimal voting rule, which maximizes the probability of the vote outcome being correct, is known to be a particular weighted majority rule that we refer to as the \emph{log-odds rule}~\cite{shapley1984optimizing,nitzan1982optimal}.
Given a vector of competences $\vec{p} = (p_1, \ldots, p_m)$, the log-odds rule assigns each expert $\rep \in \reps$ a weight $w_\rep^*$ equal to their log-odds: $w_\rep^* = \log \left( \frac{p_\rep}{1 - p_\rep} \right)$.
When $\vec{p}$ represents the true competences of the experts, the log-odds rule is optimal.
This optimality result and the nature of the binary choice motivates us to restrict our attention to weighted majority rules. Our central concern is how to assign weights to the experts based on estimates of their competences.

\paragraph{Judges' Estimates of Expert Competences}
In our model, the true competences of the experts are unknown. In order to derive the true competences we would need to assume access to the ground truth outcome, which is never revealed in our setting. Note that this is in contrast to the standard setup in online learning where the ground truth is revealed at each time step \cite{cesa1997use}. 

Any judge $\judge \in \judges$ estimating the competences of the experts is biased due to their own imperfection ($p_\judge < 1$). Every judge's competence is independent of the other judges and experts. A judge estimates an experts' competence based on how often they expect to agree. A judge with competence $p_\judge \in [0,1]$ therefore estimates the competence of expert $\rep$ as $p_{\judge\rep} = p_\judge \cdot p_\rep + (1-p_\judge)(1-p_\rep)$.

\paragraph{Aggregating Scores Into Expert Weights}
Each judge gives a score to each expert, and the scores an expert receives from the judges are then aggregated to give that expert a weight.
We assume that judges try to implement the optimal rule, assigning scores according to the log-odds rule using their perceived competences of the experts.
Hence, each judge assigns each expert a score of $w_{\judge\rep} = \log(\frac{p_{\judge\rep}}{1-p_{\judge\rep}})$.
In our model, when there are $n$ judges, the weight of an expert becomes the mean of the scores assigned to them: $w_\rep = \frac{1}{n} \sum_\judge w_{\judge\rep}$.
%

\section{Related Work}
While we focus on results for a single decision, our work is intended to be a contribution towards \emph{online group learning}, in which a set of agents (judges) collectively determines a probability distribution over potential actions. After each action, the judges individually learn the outcome of the aggregation of the expert opinions, and determine their expectation of what the future reward will be when the true rewards are only revealed after some time horizon (as opposed to being revealed after each time step). We consider performance of a single step in the action sequence where the judges all use a single strategy, although they will receive independent signals about the reward function. The correspondence between the classical online learning model and the model we propose in this paper is that the weighting of the experts and their respective competences determines a probability distribution over the vote outcomes which are the set of feasible actions~\cite{blum1998line}. In this view, a single step of the classic Multiplicative Weights algorithm for minimizing regret can be seen as a variation with a single judge where the voting rule among experts is random serial dictatorship (distributed according to the weights) instead of weighted majority voting~\cite{freund1997decision,littlestone1994weighted}.
Along the same lines, our work can be seen as a contribution towards organizational control in multi-agent learning~\cite{zhang2009integrating}.

The abstract models closest to ours are those related to generalizations of the Condorcet Jury Theorem~\cite{nitzan1994general,owen1989proving}, weighting of experts, optimal committee sizes~\cite{magdon2018mathematical,revel2021optimal}, and variants of proxy voting. In a recent paper by \citet{zhang2021tracking}, voters can transfer their votes to one another, thereby increasing the weight of the recipient's vote. This model of liquid democracy for uncovering a ground truth uses transitive delegations, so weights can be transferred multiple times along a delegation chain. In the transitive delegation model~\citet{zhang2021tracking} provide a sophisticated centralized mechanism by which the optimal graph of delegations can be constructed. We contrast this directly with our results in Section \ref{section:optimal}. The restricted nature of our score assignments is closer to that of \citet{pivato2020weighted}, which considers a process of the judges choosing the experts by an election process that also weights them and assumes there are many judges and few experts. In contrast to our approach, their finding is an asymptotic convergence result  when the number of judges tends to infinity, as is common in the literature on Condorcet Jury Theorems. However, in our work we assume a small set of judges. A similarly restricted weighting process is used by \citet{abramowitz2019flexible} who do not consider the objective of tracking a ground truth and focuses on voting on many binary issues simultaneously.

In the literature on Condorcet Jury theorems and weighting experts, the inaccessibility of expert competences has been addressed in several ways. One is to use each expert's frequency of agreement with the majority vote as a proxy for their competence, and to iteratively re-weigh them over time~\cite{grofman1983thirteen,baharad2012beyond,romeijn2011learning}.
It has also been suggested to have the experts assess each other's competences, treat this matrix as a Markov chain, and use its eigenvector values as the experts' weights~\cite{grofman1983determining} in a manner reminiscent of PageRank~\cite{page1999pagerank}.
There has also been attention paid to how group accuracy depends on the size of the group and their mean competence \cite{grofman1978judgmental,grofman1984group}, the latter of which is demonstrated in part in our empirical results in Section \ref{section:10judges}.
Most recently, \citet{baharad2022one} demonstrated empirically that when the competences of experts come from a truncated normal distribution, the optimal weighting of experts does not perform much better than an equal weighting of the experts, and the difference depends on the variance of the competence distribution. This phenomenon can be observed in our Figure \ref{fig:1judge} by comparing the central row to the top row in each of the four heatmaps for the single judge case.

\subsection{Contributions}
We begin Section \ref{section:1judge} by looking at what happens when an imperfect central judge assigns weights to experts, i.e., the case where $|\judges| = n = 1$. We demonstrate the effects both from their bias and from requiring the weights to be non-negative. In Section \ref{section:optimal} we prove that, under the right conditions, aggregating expert scores from many imperfect judges, i.e., $|\judges| = n > 1$, can reproduce the optimal log-odds rule even when none of the individual judges gives the optimal weights as their scores. However, as we argue, these conditions may not be realistic in many circumstances. Finally, in Section \ref{section:10judges}, we provide empirical results where imperfect judges score the experts and these scores are aggregated into weights. Again, the judges are inaccurate in how they estimate the competences of the experts. We look at what happens when the scores judges give must be non-negative, and when we normalize the scores of each individual judge, i.e., the contribution of each individual judge to the aggregation are all equal.

\section{Central Judge}\label{section:1judge}
We start by looking at the case where a single judge must assign weights to experts $(w_{\judge\rep} = w_\rep)$, but their estimation of the experts' competences is inaccurate as the judge does not observe the ground truth, only the output of the expert aggregation.
We look at how their perception of the experts' competences influences the overall accuracy of the system.
Next, we investigate what happens when the weights that the judge can assign to the experts are bounded from below by zero.

Recall that each expert $\rep \in \reps$ has a true \emph{competence}, or probability $p_\rep$ of voting correctly, independent of all other experts. Our central judge $\judge$ also has a competence $p_\judge$. The central judge's estimate of each expert's competence $p_{\judge\rep}$ is based on how often they tend to agree with one another: $p_{\judge\rep} = p_\judge \cdot p_\rep + (1-p_\judge)(1-p_\rep)$.
We assume that our central judge, unaware of their own imperfect competence, then attempts to implement the log-odds rule by assigning each expert a weight of $w_{\judge\rep} = \log(\frac{p_{\judge\rep}}{1-p_{\judge\rep}})$.

\begin{example}\label{example:motivation}
Suppose we have 5 experts with competences $\vec{p}_\reps = (0.6, 0.6, 0.6, 0.7, 0.9)$. The optimal weights as computed by log-odds rule are approximately $\vec{w}^*_\reps = (0.41, 0.41, 0.41, 0.85, 2.2)$. Note that with these weights, the most competent expert (0.9) receives a weight (2.2) that would make them a dictator in a weighted majority vote, since their weight is greater than all other experts combined. Hence, the accuracy under the log-odds weighting is exactly 0.9. If all the experts are weighted equally, the accuracy of the weight majority vote decreases to 0.82.
A judge with competence 0.6 would assign the experts weights of approximately $\vec{w}^{0.6}_\reps = (0.08, 0.08, 0.08, 0.16, 0.323)$. This is not equivalent to the log-odds rule because the first four experts outweigh the fifth expert alone. How high of a competence would the judge need to have to assign perceived optimal weights that correspond to the log-odds rule? The judges' competence would have to be greater than $0.962$, which is higher than any of the experts.
And yet, the judge's weighting still yields an accuracy of 0.898, which is extremely close to optimal. The question is, how much is generally lost by using sub-optimal weightings derived from the perceived competences of imperfect judges?
This example is illustrated in Figure \ref{fig:example} where we graph the overall accuracy as we sweep the judge's competence between 0.0 and 1.0.

\begin{figure}
    \centering
    \includegraphics[scale = 0.3]{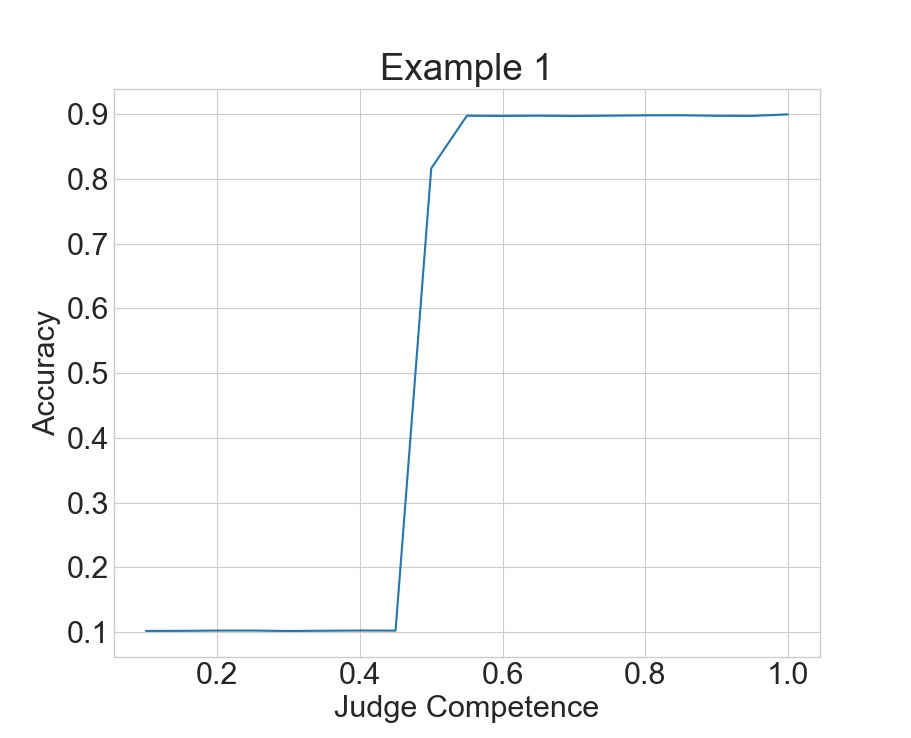}
    \caption{Accuracy of perceived optimal weightings from a single judge with the expert competences in Example \ref{example:motivation}.}
    \label{fig:example}
\end{figure}

\end{example}

To begin our empirical investigation, we simulate a setting with $m = 5$ experts and $n = 1$ judge. The $p_\judge$ value ranges from 0.1 to 1 in steps of 0.1. In the edge cases where $p_\judge = 0.5$ and $p_\judge =1.0$, the $w_{\judge\rep}$ values are all equal or correspond to the log-odds rule, respectively. The $p_\rep$ values are drawn from a truncated normal distribution $N(\mu_\reps, \sigma_\reps)$ where $\mu_\reps$ ranges from 0.1 to 0.9 in steps of 0.1, $\sigma_\reps$ ranges from 0.1 to 0.4 in steps of 0.1, and $p_\rep \in (0.1, 0.9)$ for all experts. The average accuracy of the experts' weighted majority vote is then estimated for each tuple $(p_\judge, \mu_\reps, \sigma_\reps)$. Note that the values given are how the competences were generated, not their sample mean and sample variance. This is illustrated in Figure \ref{fig:1judge}.

\begin{figure}
\begin{subfigure}{.25\textwidth}
  \centering
  \includegraphics[scale=0.25]{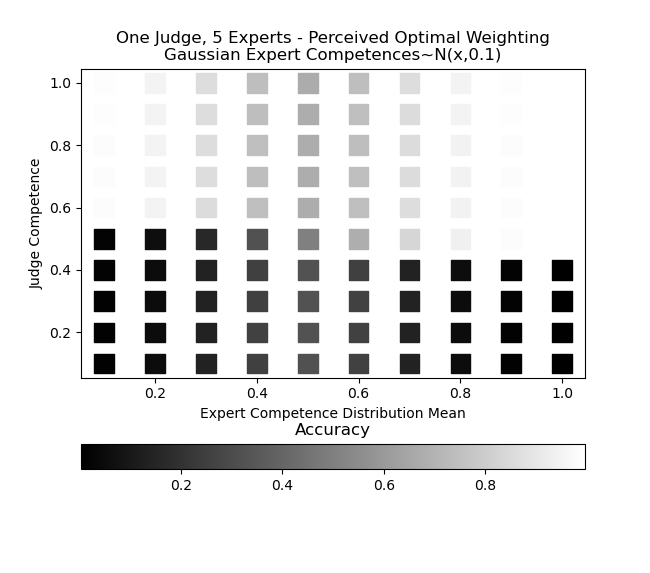}
  \label{fig:sfig1}
\end{subfigure}%
\begin{subfigure}{.25\textwidth}
  \centering
  \includegraphics[scale=0.25]{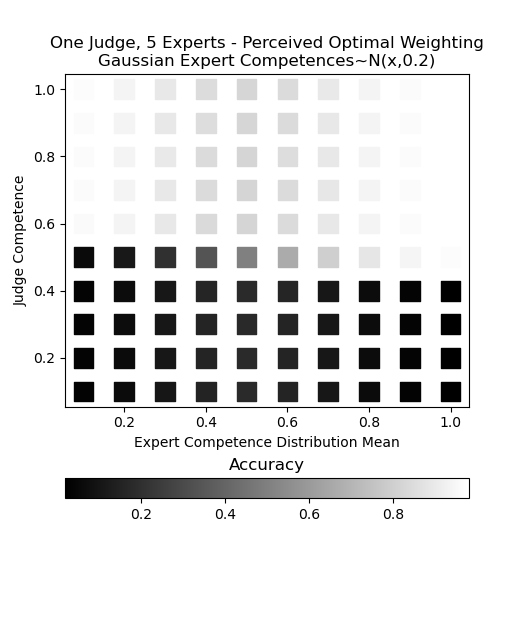}
  \label{fig:sfig2}
\end{subfigure}
\begin{subfigure}{.25\textwidth}
  \centering
  \includegraphics[scale=0.25]{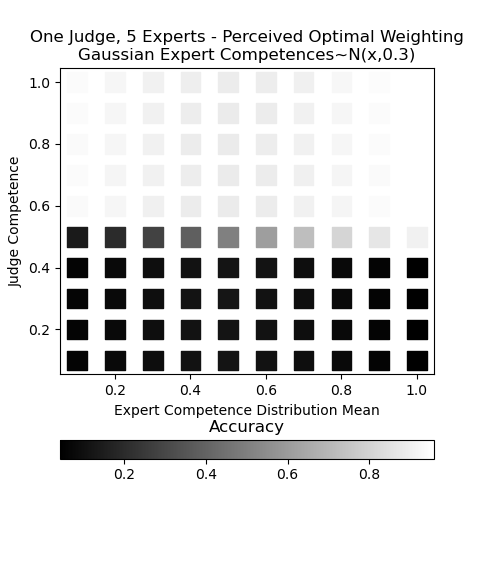}
  \label{fig:sfig3}
\end{subfigure}%
\begin{subfigure}{.25\textwidth}
  \centering
  \includegraphics[scale=0.25]{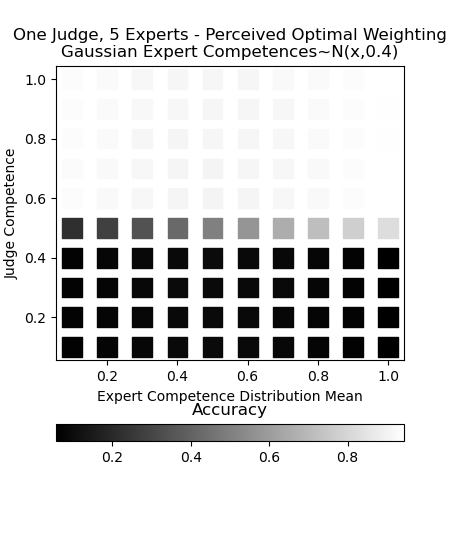}
  \label{fig:sfig4}
\end{subfigure}
\caption{Heatmaps of accuracy for single judge competence for Gaussian distributions of 5 expert competences with variances \{0.1, 0.2, 0.3, 0.4\}.}
\label{fig:1judge}
\end{figure}


The accuracy increases smoothly as $\mu_\reps$ increases, but there is a marked transition where $p_\judge$ goes from competence below 0.5 to above 0.5. The higher $\sigma_\reps$ the more marked the transition. Intuitively, higher $\sigma_\reps$ improves accuracy when the expert weights are `closer' to optimal $(p_\judge = 1)$ and further from equality $(p_\judge = 0.5)$.

\section{Optimal Distributed Weighting}\label{section:optimal}
Moving to the multi-agent setting, we now turn our attention to the potential for improved accuracy when there are multiple judges, i.e., $|\judges| = n > 1$. 
Recall that $\reps$ is a set of $m$ experts indexed by $\rep \in \reps$.
Every expert $\rep \in \reps$ has a competence $p_\rep \in (0,1)$ reflecting their probability of voting correctly, independently of all other experts and judges.
The rule that maximizes the probability of selecting the correct alternative is the log-odds rule in which every expert is assigned a weight equal to the log-odds of their competence: $w^*_\rep = \log \left( \frac{p_\rep}{1 - p_\rep} \right)$~\cite{shapley1984optimizing,nitzan1982optimal}.
However, in our work we do not assume that the experts' competences are known.
We cannot compute $w^*_\rep$ directly if we do not know $p_\rep$.

Each judge $\judge$ assigns each expert $\rep$ a score that they believe is their Bayesian optimal weight $w_{\judge\rep} = \log \left( \frac{p_{\judge\rep}}{1 - p_{\judge\rep}} \right)$.
The average (\emph{arithmetic mean}) of these scores becomes the weight of the expert: $w_\rep = \frac{1}{n} \sum\limits_\judge w_{\judge\rep}$.

We prove that when the geometric mean of the judges' estimates of experts competence odds are the experts true competence odds, i.e., $\left( \frac{p_\rep}{1-p_\rep} \right) = \left(\prod\limits_\judge \frac{p_{\judge\rep}}{1-p_{\judge\rep}} \right)^\frac{1}{n}$, all experts are assigned their Bayesian optimal weights $w_\rep = w^*_\rep$. Remarkably, this does not require any of the individual judges to know the experts' true competences.

\begin{theorem}\label{theorem:optimal}
If each judge assigns each expert a score equal to the log-odds of their perceived competence, and the geometric mean of the judges' estimates of each expert's competence odds is the expert's true odds, then the weighted majority rule using judges' average scores to weight each expert is exactly the optimal log-odds rule.
\end{theorem}

\begin{proof}
We begin by assuming that the judges give each expert a score of $w_{\judge\rep} = \log \left( \frac{p_{\judge\rep}}{1-p_{\judge\rep}} \right)$, corresponding to what they believe the optimal weight of that expert to be, and we take the average as the weight of the expert. 

\begin{align}
    w_\rep = \frac{1}{n} \sum\limits_\judge w_{\judge\rep} & = \frac{1}{n} \sum\limits_\judge \log \left( \frac{p_{\judge\rep}}{1-p_{\judge\rep}} \right)\\
    w_\rep & = \frac{1}{n} \log \left( \prod_\judge \frac{p_{\judge\rep}}{1-p_{\judge\rep}} \right)\\
    w_\rep & = \log \left( \left( \prod_\judge \frac{p_{\judge\rep}}{1-p_{\judge\rep}} \right)^\frac{1}{n} \right)
\end{align}

Now we assume the geometric mean of judges' estimates of the experts' competence odds is correct. We assume $\left( \frac{p_\rep}{1-p_\rep} \right) = \left(\prod\limits_\judge \frac{p_{\judge\rep}}{1-p_{\judge\rep}} \right)^\frac{1}{n}$. Therefore, 

\begin{align}
    w_\rep & = \log \left( \frac{p_\rep}{1-p_\rep} \right) = w^*_\rep
\end{align}

\end{proof}

\begin{corollary}\label{corollary:optimal}
If the geometric mean of judge estimates of competence odds is off by some multiplicative factor $\alpha$ for some expert, then the error of that expert's weight is only $\log(\alpha)$.
\end{corollary}

\begin{proof}
In the proof above, assume instead that $\alpha \left( \frac{p_\rep}{1-p_\rep} \right) = \left(\prod\limits_\judge \frac{p_{\judge\rep}}{1-p_{\judge\rep}} \right)^\frac{1}{n}$. Then,

\begin{align}
    w_\rep & = \log \left( \alpha \cdot \frac{p_\rep}{1-p_\rep} \right) = w^*_\rep + \log(\alpha)
\end{align}

\end{proof}

\autoref{theorem:optimal} and \autoref{corollary:optimal} provide us with a starting point for our investigation of distributed weighting of experts. Together they state that if all judges individually form personal estimates of the experts' competences, then so long as their collective estimate is reasonably accurate -- the geometric mean of the odds implied by the weights is within a small multiplicative factor of the true odds -- the weights they assign to the experts by averaging their scores will be ``close'' to the Bayesian optimal weights.
This corollary is promising because any set of weights defines a collection of subsets, or ``winning coalitions", such that the outcome is guaranteed if all experts in the subset vote the same way. Altering the weights only changes the weighted majority rule if the set of winning coalitions changes. 

However, there are clear shortcomings to this result.
The first is that for the result to hold judges must be able to assign negative scores to experts, which may not be desirable in many circumstances. The second issue is that if judges express complete certainty, by privately estimating the competence of an expert as either 1 or 0 (always correct or always incorrect), then the expert's score is undefined. The third issue is that the scores judges are able to assign can be arbitrarily large or small even when they are defined. There is no bound on how large a positive or negative score could be, so a single judge assigning huge scores could completely determine the outcome.

These shortcomings related to the practicality of the judges reporting space motivates the study of limits on the judges' scores. There are a few ways to bound the scores that judges can assign to experts that address these issue. 
The simplest is by ensuring ``one person, one vote," i.e., normalization, so that each judge gets a budget of points that they can distribute among the experts to construct their scores: $\forall i \sum\limits_{j \in \reps} w_{\judge\rep} = 1$.

\section{Distributed Weighting}\label{section:10judges}
We now turn our attention back to the empirical study of the distributed weighting of experts with $n = 10$ judges rather than $n = 1$. Based on the results of Section \ref{section:1judge}, we consider a low variance (0.1) and high variance (0.4) condition for the competences of both experts and judges. For each condition, we look at the loss of accuracy when scores given by the judges are restricted to being non-negative and when they must be normalized.

Figure \ref{fig:10judges_unrestricted} illustrates the case of unrestricted scores for 50k trials. We see a pattern very similar to what we observed with a single judge in Figure \ref{fig:1judge}. Notably, when $\sigma_\reps$ is high but $\sigma_\judges$ is low, we see the marked phase transition, where any competent judge $(p_\judge > 0.5)$ seems to yield scores similar to the optimal weights (Figure \ref{fig:10judges_unrestricted} bottom left). However, when both sets of agents are in the high variance condition (Figure \ref{fig:10judges_unrestricted} bottom right), $\mu_\reps$ seems to hardly matter at all in comparison to $\mu_\judges$. When $\sigma_\reps$ is low (top row), the effects from greater expert competence are more pronounced, particularly when $\sigma_\judges$ is low too.

\begin{figure}[h!]
\begin{subfigure}{.25\textwidth}
  \centering
  \includegraphics[scale=0.25]{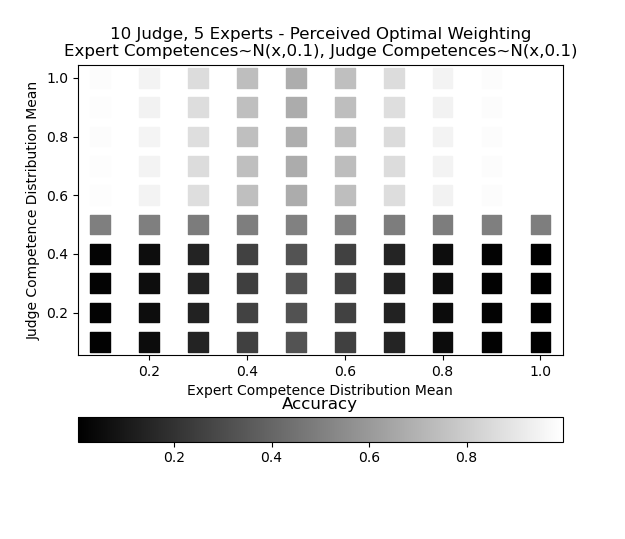}
\end{subfigure}%
\begin{subfigure}{.25\textwidth}
  \centering
  \includegraphics[scale=0.25]{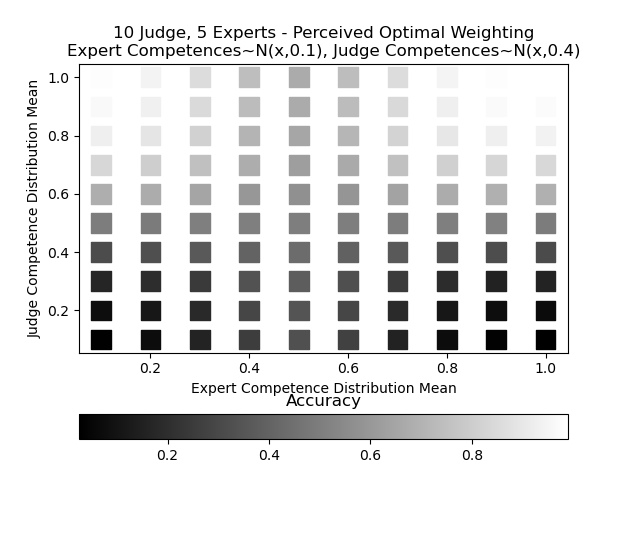}
\end{subfigure}
\begin{subfigure}{.25\textwidth}
  \centering
  \includegraphics[scale=0.25]{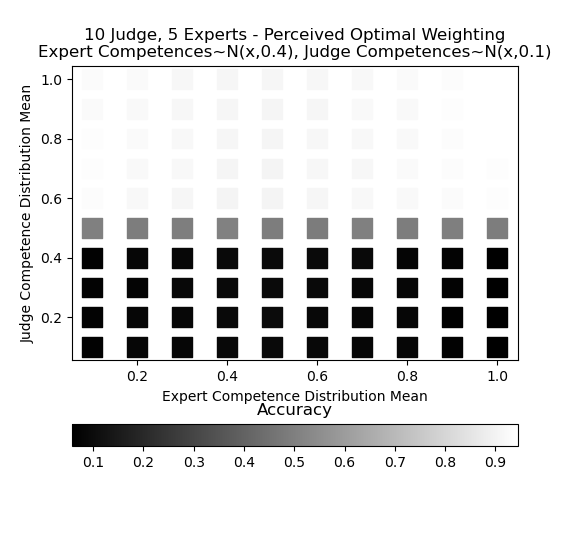}
\end{subfigure}%
\begin{subfigure}{.25\textwidth}
  \centering
  \includegraphics[scale=0.25]{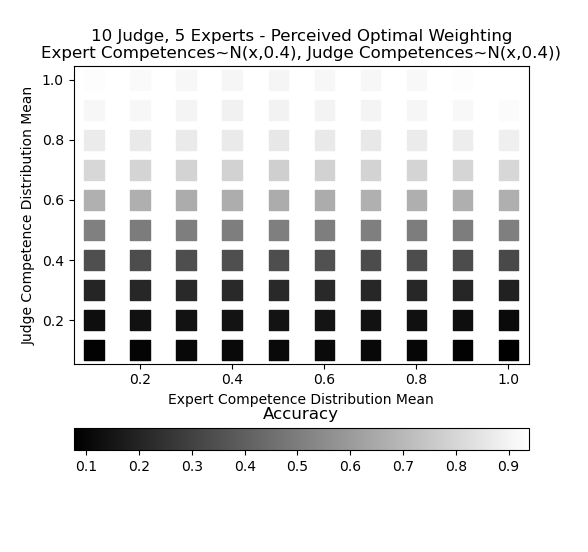}
\end{subfigure}
\caption{Heatmaps of accuracy for 10 judges and 5 experts with competence variances in $\{0.1, 0.4\}$.}
\label{fig:10judges_unrestricted}
\end{figure}

Restricting the weights to be non-negative has a significant, observable impact on accuracy, as shown in Figure \ref{fig:nneg}. When $\sigma_\reps$ is low and weights are non-negative (\ref{fig:nneg} top row), the effect of changes in $\mu_\judges$ is dwarfed by the impact of changes in the expert mean competence. 

Surprisingly, normalizing the weights from each judge causes very little loss in accuracy compared to restricting the weights to be non-negative. This holds true in all four \{high, low\} $\times$ \{high, low\} variance conditions in Figure \ref{fig:10judge_normalized}.

\begin{figure}[h!]
\begin{subfigure}{.25\textwidth}
  \centering
  \includegraphics[scale=0.25]{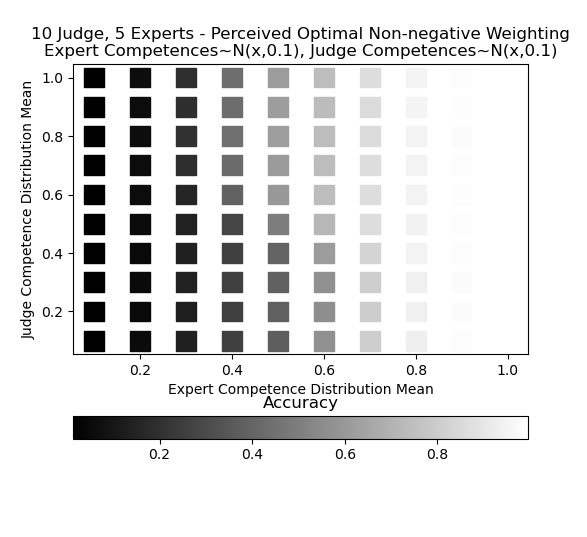}
\end{subfigure}%
\begin{subfigure}{.25\textwidth}
  \centering
  \includegraphics[scale=0.25]{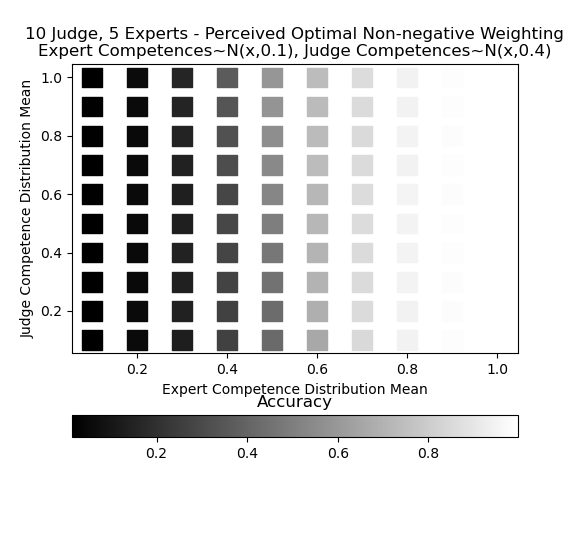}
\end{subfigure}
\begin{subfigure}{.25\textwidth}
  \centering
  \includegraphics[scale=0.25]{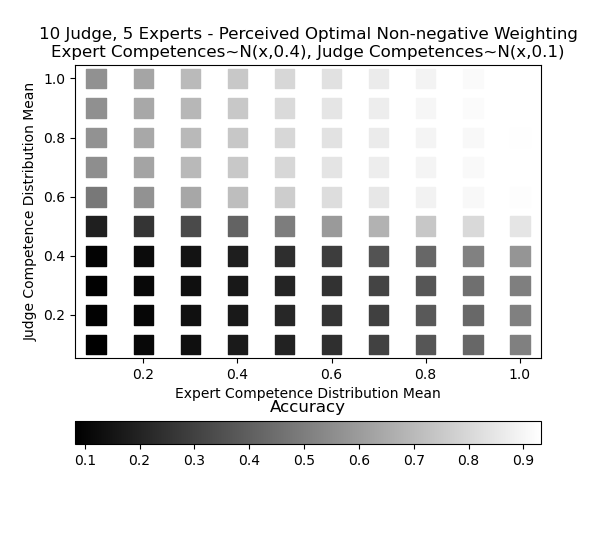}
\end{subfigure}%
\begin{subfigure}{.25\textwidth}
  \centering
  \includegraphics[scale=0.25]{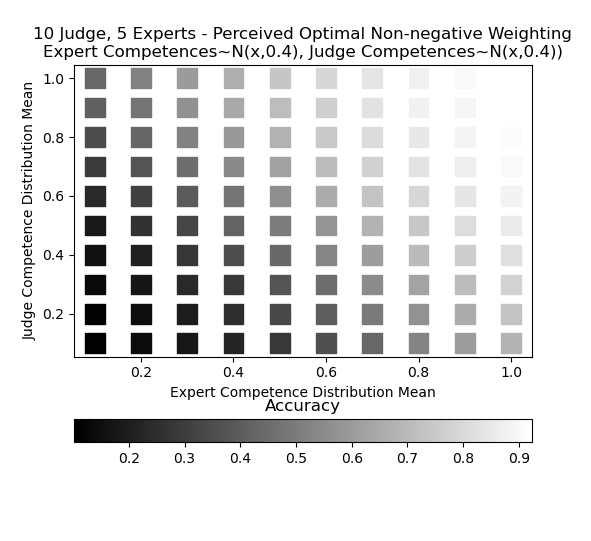}
\end{subfigure}
\caption{Heatmaps of accuracy for 10 judges and 5 experts with competence variances in $\{0.1, 0.4\}$ and non-negative weights.}
\label{fig:nneg}
\end{figure}

\begin{figure}[h!]
\begin{subfigure}{.25\textwidth}
  \centering
  \includegraphics[scale=0.25]{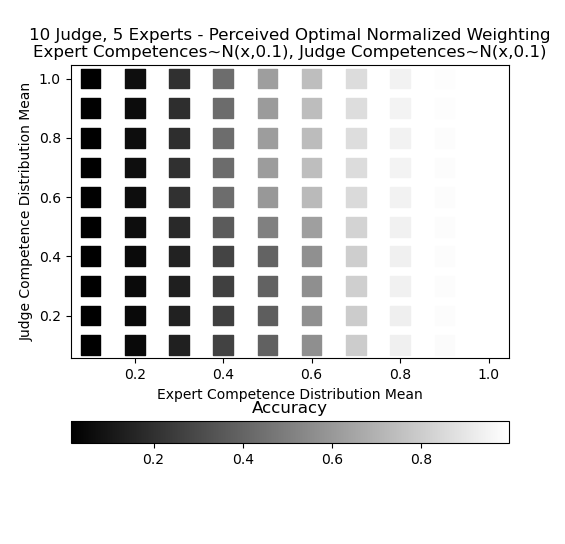}
\end{subfigure}%
\begin{subfigure}{.25\textwidth}
  \centering
  \includegraphics[scale=0.25]{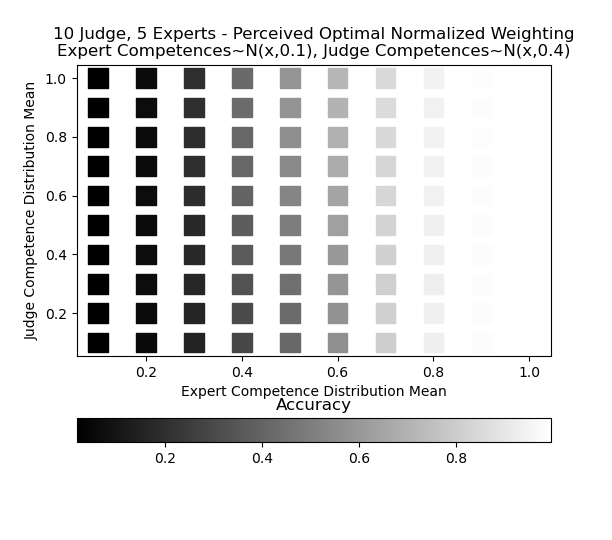}
\end{subfigure}
\begin{subfigure}{.25\textwidth}
  \centering
  \includegraphics[scale=0.25]{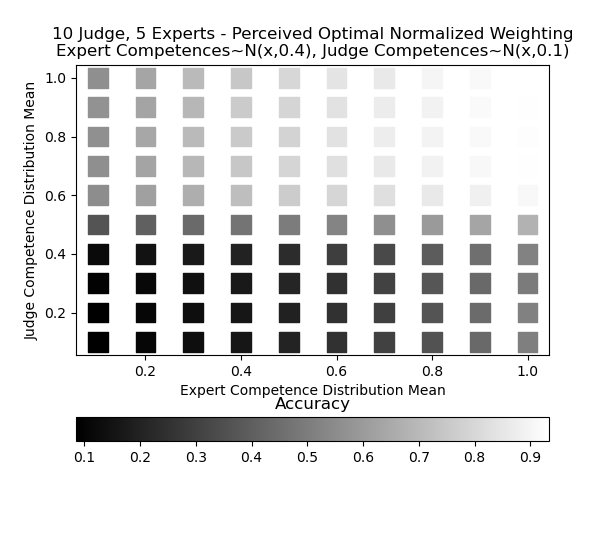}
\end{subfigure}%
\begin{subfigure}{.25\textwidth}
  \centering
  \includegraphics[scale=0.25]{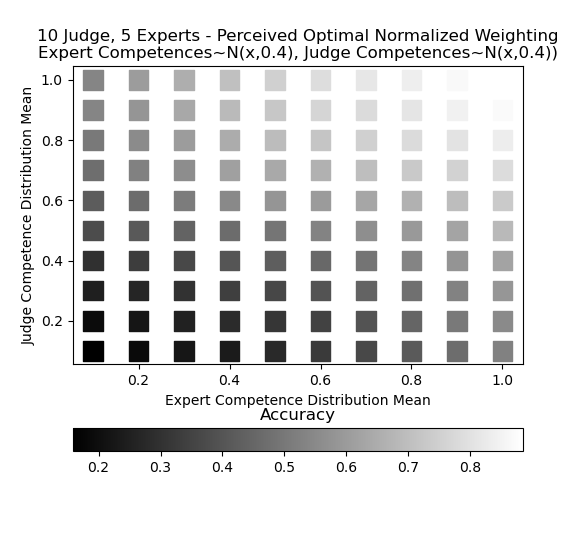}
\end{subfigure}
\caption{Heatmaps of accuracy for 10 judges and 5 experts with competence variances in $\{0.1, 0.4\}$ and normalized weights.}
\label{fig:10judge_normalized}
\end{figure}

\pagebreak

\section{Discussion}
Building on the literature of weighting experts, we have introduced a model in which a set of judges assesses the competence of a set of experts, and weights them accordingly in a distributed fashion before the experts use a weighted majority vote to make a decision. When the scores from independent judges are averaged to give each expert their weight, we have given sufficient conditions for the weights to be optimal even when no individual judge knows the true competence of any expert or the ground truth. Our empirical results show (1) judges' perception of the experts' competences leads to sub-optimal weightings that produce lower accuracy but compete well with the optimal log-odds rule in many cases, (2) the variance in expert and judge competences determines the relative effect sizes of changes in the mean competences, and (3) requiring weights to be non-negative leads to a moderate loss of accuracy, but normalizing the weights causes very little additional loss.

\section{Future Work}
We leave many avenues open to further exploration. There are many alternative ways in which judges might estimate experts' competences and assign their weights, and different distributions of competences may be relevant to different applications. We did not begin to address here any correlation between the competences, weights, or votes of the judges and experts~\cite{shapley1984optimizing}.
Following the sensor example, we would also like to assess the performance of these distributed judge-expert systems when all judges are not always available; similar to the delegation rate in some delegative voting models~\cite{abramowitz2019flexible}.
Also in line with the voting literature would be the consideration of multiple binary issues simultaneously~\cite{baharad2011distilling, grofman1984group}.

Characterizing the equilibria when judges and experts are strategic, in the manner of \citet{zhang2021tracking}, is another promising direction which would be complicated by an understanding of how judges can learn to optimize their weightings over time given their signals. We hope that our results are seen as a small step towards a deeper understanding of online multi-agent learning.

Another line of thought is in the design of judge-expert systems with resource constraints. For instance, if one has a fixed number of agents and knows something about the distribution of their competences, how does one optimally divide them into judges and experts? And how does the distributed weighting of experts compete with models of weighting experts based on their voting histories or having the experts all weight each other~\cite{grofman1983determining}?

\section*{Acknowledgements}
Nicholas Mattei was supported by NSF Awards IIS-RI-2007955, IIS-III-2107505, and IIS-RI-2134857, as well as an IBM Faculty Award and a Google Research Scholar Award.

\bibliographystyle{ACM-Reference-Format}
\bibliography{main}

\end{document}